\documentclass[11pt]{article} 

\usepackage{hyperref}
\usepackage{url}
\usepackage{amstext,amssymb,amsmath}
\usepackage{amsthm}

\long\def\icml#1{}
\long\def\nips#1{}
\long\def\arxiv#1{#1}
\icml{\usepackage{icml2019}}
\nips{\usepackage{neurips_2019}}
\arxiv{\usepackage{halffull}}

\newtheorem{theorem}{Theorem}

\newtheorem{lemma}{Lemma}

\newtheorem{corollary}{Corollary}

\newcommand{\pder}[2]{\frac{\partial #1}{\partial #2}}

\newcommand{\wt}{\widetilde}
\newcommand{\R}{\mathbb{R}}
\newcommand{\M}{\mathcal{M}}

\nips{
\title{Depth creates no more spurious local minima}
\author{Li Zhang\vspace*{0.1in}\\
Google Research\\
\texttt{liqzhang@google.com}
}}

\begin{document}

\icml{

\twocolumn[\icmltitle{Depth creates no more spurious local minima}



\icmlsetsymbol{equal}{*}

\begin{icmlauthorlist}
\icmlauthor{Li Zhang}{google}
\end{icmlauthorlist}

\icmlaffiliation{google}{Google Inc}

\icmlkeywords{Optimization, local minimum}

\vskip 0.3in
]}

\arxiv{
\title{Depth creates no more spurious local minima}
\author{Li Zhang\vspace*{0.1in}\\
Google Research\\
liqzhang@google.com}
\maketitle
}

\nips{
\maketitle
}




\begin{abstract}
We show that for any convex differentiable loss, a deep
linear network has no spurious local minima as long as it is true for the two
layer case. This reduction greatly simplifies the study on the existence of
spurious local minima in deep linear networks. When
applied to the quadratic loss, our result
immediately implies the powerful result by Kawaguchi~\cite{kawaguchi2016}. Further, with the work in~\cite{zhou2018}, we can remove all the assumptions in~\cite{kawaguchi2016}.
This property holds for more general ``multi-tower'' linear networks too.
Our proof builds on~\cite{laurent2018} and develops a new perturbation argument
to show that any spurious local minimum must have full rank, a structural
property which can be useful more generally.
\end{abstract}

\section{Introduction}\label{sec:intro}

One major mystery in deep learning is that deep neural networks do not seem
to suffer from spurious local minima. Understanding this mystery
has become one of the most important topics in the machine learning theory.
We usually attribute the existence of spurious local minima to two factors:
the non-linearity in activation and the large depth of the network. 
In this paper, we show that depth does
not create more spurious local minima in linear networks. More precisely,
we show that for any convex differentiable loss function, any spurious
local minima in a deep linear network should already be present in a two
layer linear network\footnote{Here depth is defined as the number of matrices in the parameter.
Two layer network is the same as one hidden-layer network in some literatures.}.
Such reduction greatly simplifies the study on
the important question about the existence of spurious local minima in deep
linear networks. When applied to the quadratic loss, it leads to the first
unconditional proof that there is no local minima in deep linear networks.
In addition, our proof reveals a non-degeneracy property of spurious local
minima in linear networks, i.e. any spurious local minimum has to
have full rank. This property can also be useful for analyzing two layer networks.

Baldi and
Hornik~\cite{baldi1989} started the investigation on the existence of spurious local minima in linear networks. They showed that, under mild assumptions, for
quadratic loss, two layer linear networks do not have
spurious local minima. They also conjectured it is true for deep linear
networks. This conjecture is only proved almost thirty years later by Kawaguchi~\cite{kawaguchi2016}.
For the special case of linear residual networks, Hardt and Ma~\cite{hardt2017} showed that
there are no spurious local minima through a simpler argument.

While most existing work have been on quadratic loss functions,
Laurent and von~Brecht~\cite{laurent2018} showed a surprisingly
general result for any convex differentiable loss. In~\cite{laurent2018},
the authors consider the special linear networks which have no bottlenecks,
i.e. when the narrowest layer is on the either end. They showed that for any convex differential
loss function, there is no spurious local minima in such networks.
In addition to its generality, the proof in \cite{laurent2018} is quite
intuitive through a novel perturbation argument. However, their special
cases excludes networks with
bottleneck layers, commonly
used in the practice and studied in the literature~\cite{baldi1989,kawaguchi2016}.

We build on the work in~\cite{laurent2018} and further develop the
technique to show that for general deep linear networks,
whether there are spurious local minima is reduced to the two layer case.
\begin{theorem}\label{thm:main}
Given any convex differentiable function $f:\R^{m\times n}\to\R$. For any $k\geq 2$, let $L_k(M_1, \ldots, M_k) = f(M_k\cdots M_1)$ where  $M_i\in \R^{d_i\times d_{i-1}}$ for $1\leq i\leq k$ with $d_k=m$ and $d_0=n$.
Let $d=\min_{0\leq i\leq k}d_i$. Define $L_2(A, B)=f(AB)$ for $A\in\R^{m\times d}, B\in\R^{d\times n}$.
Then $L_k$ has no spurious local minima iff $L_2$ has no spurious local minima.
\end{theorem}

We emphasize that in the above theorem, $f$ depends on both
the data and the loss function. Hence the reduction is instance specific and
does not depend on the global property of a family of loss functions. Such reduction to two layer network greatly simplifies the
study on the existence of spurious local minima in deep linear
networks. 

To prove Theorem~\ref{thm:main}, we show a key structural property of
local minima, namely, any spurious local minimum, when ``broken'' at the
bottleneck, must have full rank. When specialized to two layer case,
it basically implies that a spurious local minimum cannot be rank-deficient,
immediately covering such cases which would otherwise require onerous analysis
such as those in~\cite{baldi1989,zhou2018}.
\begin{theorem}\label{thm:two}
With the same notation as in Theorem~\ref{thm:main}, if $A,B$ is a spurious local minimum of $L_2$, then both $A,B$ have full rank, i.e. rank $d$.
\end{theorem}

As an application of Theorem~\ref{thm:main}, for quadratic loss, Theorem~\ref{thm:main}, together with~\cite{baldi1989},
immediately implies the main result in~\cite{kawaguchi2016}. We can further remove all
the assumptions needed in~\cite{baldi1989,kawaguchi2016}, using the result of~\cite{zhou2018} (Theorem 2(1)), hence providing the first unconditional proof of the non-existence of spurious local minima in deep linear networks for quadratic loss functions. Below $\|\cdot\|$ denotes the Frobenius norm.
\begin{corollary}\label{cor:square}
For any $X\in \R^{d_0\times n}, Y\in\R^{d_k\times n}$, let $L(M_1,\ldots,M_k)=\|M_k\cdots M_1X-Y\|^2$ . Then $L$ has no spurious local minima.
\end{corollary}

Theorem~\ref{thm:main} can be further generalized to ``multi-tower'' linear networks. Define a multi-tower linear
network as the sum of multiple deep linear networks (towers), i.e. $M_{1k_1}\cdots M_{11} + \ldots + M_{sk_s}\cdots M_{s1})$,
where $M_{ij} \in \R^{d_{i,j}\times d_{i,j-1}}$ with $d_{i,k_i}=m$ and $d_{i,0}=n$.
For $1\leq i\leq s$, let $b_i=\min_j d_{i,j}$ denote the bottleneck size of each tower $i$. Write $b=\sum_i b_i$. 
\begin{corollary}\label{cor:multi}
For any differentiable convex loss $f$, a multi-tower linear network has no spurious local minima iff the linear network $AB$, where $A\in\R^{m\times b}$, $B\in\R^{b\times n}$, has no spurious local minima. Moreover, if $b \geq m, n$, there is no spurious local minima in this network.
\end{corollary}

\paragraph{Overview of the proof.} The interesting case is when the bottleneck, i.e. the
narrowest layer, is a middle layer (otherwise it is already covered by~\cite{laurent2018}). In this
case, we split the network at the bottleneck into two parts
and regard it as a two ``super-layer''
network where each super-layer is parameterized as a product of multiple layers.
We first observe that by applying the result in~\cite{laurent2018}, any local minimum
is a critical point of the two layer network. Further we show that, unless the solution
is already a global optimum, the multi-layer parameterization of each super-layer is
``non-degenerate'' so we can always
perturb a critical point locally at the super-layer level to reduce it to the
two layer case. The main technical part is the non-degeneracy property, which
we prove by developing a new rank one perturbation argument motivated by those in~\cite{laurent2018}.

\subsection{Related work}

In \cite{baldi1989}, it is shown that, under mild assumptions on data,
two layer linear network with quadratic
loss has no spurious local minima. This is probably the first positive result on this long line of
investigation. Kawaguchi~\cite{kawaguchi2016} showed that it holds for deep linear network
too. The tour de force proof in~\cite{kawaguchi2016} works by examining the Hessian using powerful
tools from the matrix theory. There have been much subsequent work to simplify and generalize the result. For example,
\cite{lu2017depth} came up with a different argument. \cite{yun2018,yun2019} showed simpler
arguments for special cases and considered more general non-linear networks. \cite{hardt2017}
showed that under certain assumptions, there might not even be stationary point in the deep
linear residual network. \cite{venturi2018spurious} defined a notion of spurious valleys
and showed that for quadratic losses, there is no spurious valley in deep linear networks. \cite{venturi2018spurious} was able to remove all the assumption in~\cite{baldi1989} under this weaker notion. The mild assumption in~\cite{baldi1989}, which was also needed in Kawaguchi's proof, was
removed by~\cite{zhou2018}, which leads to our Corollary~\ref{cor:square}.

\cite{laurent2018} considers the special case of linear networks with the narrowest layer on the either end. It uses a novel perturbation argument to show that for any convex differentiable loss,
there is no spurious local minimum in such network. However, the special case considered in~\cite{laurent2018} excludes networks through low rank approximation such as auto-encoders. But it is really the intuitive yet powerful result in~\cite{laurent2018} that motivated this work.

There have been recent studies on the gradient descent convergence on the deep linear
networks~\cite{arora2018a,arora18,bartlett2018}. It has been shown~\cite{arora18} that,
under certain conditions, increasing the depth of linear networks can speed up the
convergence, which is another positive property of deep linear networks.

There have been much work~\cite{x3,x4,x5,x7,x2,x1,x6} on studying the optimization landscape and convergence of non-linear networks. They focus mostly on shallow networks with over-parameterized wide layers.

\section{Preliminaries}

We define notations used through the paper. We state some simple facts and the main theorem 
from~\cite{laurent2018} which we need in our proof.

\paragraph{Deep linear networks.}
Denote by $\R^{m\times n}$ all the matrices with $m$ rows
and $n$ columns. For $1\leq i\leq k$, let $M_i \in \R^{d_i \times d_{i-1}}$ . For $i\geq j$, denote by
$M_i \cdots M_j$ the matrix product of $M_i \cdot M_{i-1} \ldots \cdot M_j$. A (deep) linear
network with parameters $M_1, \cdots, M_k$ is defined as $\Phi(x) = M_k \cdots M_1 x$.
We call $k$ the depth of the network and $d_0, d_k$ the input and the output dimensions, respectively.
Define $d=\min_{0\leq i\leq k} d_i$ be the narrowest width. We say a network has
a bottleneck if both $d_0>d$ and $d_k>d$.
A multi-tower linear network is defined as the sum of multiple linear networks (towers) with the same input and output dimensions. 

\paragraph{Empirical loss.}
Given training data $D$ which consist
of examples of pairs of $x, y$ where the input feature vector $x\in\R^{d_0}$, and the label $y$ in some arbitrary set,
we wish to minimize the total loss:
\[ L(M_1, \cdots, M_k; D) = \sum_{(x, y)\in D} f_y (\Phi(x)) = \sum_{(x, y)\in D} f_y(M_k \cdots M_1 x)\,.\]

Define $f(A) = \sum_{(x, y)\in D} f_y(Ax)$. Then $L(M_1, \cdots, M_k; D) = f(M_k \cdots M_1)$. 
If $f_y$'s are all convex differentiable functions\footnote{In practice, $f_y$'s are typically convex. They are usually differentiable, and if not, can be smoothly approximated. For example the hinge loss can be approximated by the modified Huber loss~\cite{zhang2004solving}.}, then clearly $f$ is convex differentiable too. Below we omit
$D$ and consider the loss function $L:\R^{d_1\times d_0} \times \ldots \times \R^{d_k \times d_{k-1}}\to\R$
where $L(M_1, \cdots, M_k) = f(M_k \cdots M_1)$ for some $f:\R^{d_k\times d_0}\to \R$.

\paragraph{Derivative.} Denote by $\pder{L}{M}$ the matrix
form of the partial derivative of $L$ with respect to $M$.  If $L$ has only
one variable $M$, we write $L'=\pder{L}{M}$. For simplicity,
we sometimes abuse the notation by using the same symbol for the variable and the
value and omit the value. If $L(X, Y) = f(XY)$, by the chain rule, $\pder{L}{X}=f'(XY)Y^T$ and
$\pder{L}{Y}=X^Tf'(XY)$.

\paragraph{Local minimum.}
For any loss function $L$, $M_1, \ldots, M_k$ is a local minimum of $L$ if there exists an open ball $B$,
in Frobenius norm, centered at $M_1, \ldots, M_k$ such that $L(M_1, \ldots, M_k)\leq L(M_1', \cdots, M_k')$ for any
$(M_1', \cdots, M_k') \in B$. A local minimum is called \emph{spurious} if it is not a global minimum.
If $L$ is differentiable, then any local minimum is a critical point of $L$. In particular,
if $L(M_1, \cdots, M_k) = f(M_k \cdots M_1)$, then $M_1, \cdots, M_k$ satisfy
that $\pder{L}{M_i} = (M_k\cdots M_{i+1})^Tf'(M_k\cdots M_1) (M_{i-1} \cdots M_1)^T = 0$

\bigskip

We need the following theorem from~\cite{laurent2018}:
\begin{theorem}\label{thm:all}
Let $L_k(M_1, \ldots, M_k) = f(M_k\cdots M_1)$ where $f:\R^{d_k \times d_0}\to\R$ is a convex differentiable function.
If there is no bottleneck, i.e. $d_0$ or $d_k=\min_{0\leq i\leq k} d_i$, then any local minimum of $L_k$ is a global minimum of $f$.
\end{theorem}

\section{Proofs}

With the above preparation, we will now prove Theorem~\ref{thm:main}.
If the network we consider has no bottleneck, i.e. $d=d_k$ or $d=d_0$, then Theorem~\ref{thm:all} 
immediately implies that all the local minima for $L_k$ are global minima of $f$ so the statement is vacuously true.
Below we consider the case when $d=d_j$ for some $0<j<k$. 
Let $A=M_k\cdots M_{j+1}$ and $B=M_j\cdots M_1$. The following is the main technical claim of the paper.
\begin{lemma}\label{lem:full}
If $M_1, \ldots, M_k$ is a local minimum, then either $f'(AB)=0$ or $A,B$ both have rank $d$.
\end{lemma}

We first show that the above lemma implies Theorem~\ref{thm:main}. 
\begin{proof}\textbf{(Theorem~\ref{thm:main})}
If $L_2$ has spurious local minima, then $L_k$ has too.
Since $d_j=\min_{0\leq i\leq k} d_i$, 
for any $A\in\R^{d_{k}\times d_j}$ and $B\in \R^{d_j\times d_0}$, we
can easily construct matrices
$M_i\in \R^{d_i \times d_{i-1}}$ for $1\leq i\leq k$ such
that $M_k\cdots M_{j+1} = A$, and $M_j\cdots M_1=B$.
If $A,B$ is a spurious local minimum of $L_2$, then clearly $M_1, \cdots, M_k$
is a spurious local minimum of $L_k$.

The other direction is implied by Lemma~\ref{lem:full}. This implication has been used
before multiple times~\cite{kawaguchi2016,lu2017depth,yun2018}. Here we include the easy proof
for completeness. Suppose
that $M_1, \cdots, M_k$ is a local minimum of $L_k$. Then by Lemma~\ref{lem:full},
either $f'(AB)=0$ or $A,B$ both have rank $d$.
If $f'(AB)=0$,
then $AB=M_k\cdots M_1$ is a global minimum of $f$ because $f$ is convex.
Hence $M_1, \ldots, M_k$ is a global minimum of $L_k$ too.

In the other case, $A$ and $B$ both have full rank $d$. We show that any local perturbation to
$A$ (resp. $B$) can be performed by local perturbation to $M_k$ (resp. $M_1$).
If $A=M_k\cdots M_{j+1}$
has rank $d$, then $A_1=M_{k-1}\cdots M_{j+1}\in
\R^{d_{k-1}\times d}$ has rank $d$ too because $d\leq d_{k-1}$. Then for any 
$D\in \R^{d_k\times d}$, there exists $D_1\in \R^{d_k\times d_{k-1}}$ such that 
$D_1 A_1 = D$. Hence $(M_k+D_1) A_1=M_kA_1 + D_1
A_1 = A + D$. This implies any local perturbation to $A$ can be done
through local perturbation to $M_1$. More precisely, there exists a
constant $c>0$, such that for any $D\in \R^{d_k \times d}$, there
exists $D_1\in R^{d_k \times d_{k-1}}$ with $\|D_1\|\leq c\|D\|$
and $D_1 A_1=D$. Same is true for $B$. This implies that if
$L_k(M_1, \ldots, M_k)$ is minimum in an open ball of radius $\delta$
centered at $M_1, \ldots, M_k$, then $L_2(A,B)$ is minimum in an
open ball of radius $\delta/c$ centered at $A, B$. 
Hence if $M_1, \ldots, M_k$ is a local minimum of $L_k$,
then $A,B$ is a local minimum of $L_2$. If $L_2$ has no
spurious local minima, $A,B$, hence $M_1, \ldots, M_k$, is a global minimum.
\end{proof}

In the above proof, we actually showed that if $M_1, \ldots, M_k$ is a spurious local minimum
of $L_k$, then $A=M_k\cdots M_{j+1},B=M_j\cdots M_1$ is a spurious local minimum of $L_2$. That is,
every spurious local minima of $L_k$ can be directly mapped to a spurious local minimum
of $L_2$, hence the title of the paper.

Lemma~\ref{lem:full} directly implies Theorem~\ref{thm:two}.
\begin{proof}\textbf{(Theorem~\ref{thm:two})}
Consider the case of $k=2$ and $j=1$. Then we have $A=M_2\in\R^{d_2\times d_1}$ and $B=M_1\in\R^{d_1\times d_0}$ with $d_1 < \min(d_0, d_2)$. If $A,B$ is a spurious
local minimum of $L_2$, then $f'(AB)\neq 0$ because otherwise they would have been a global minimum of $f$.
By Lemma~\ref{lem:full}, we have that both $A,B$ are of rank $d_1$, i.e. they both have full rank because $d_1<d_0, d_2$.  
\end{proof}

To prove Lemma~\ref{lem:full}, we first observe that
\begin{lemma}\label{lem:critical}
If $M_1, \ldots, M_k$ is a local minimum of $L_k$, then $\pder{L_2}{A}(A,B)=0, \pder{L_2}{B}(A,B)=0$.
\end{lemma}
\begin{proof}
Define $g_B(X)=L_2(X, B)=f(XB)$. Clearly $g_B$ is convex and differentiable too. Let
$\wt{L}_B(M_{j+1}, \ldots, M_k) = g_B(M_k \cdots M_{j+1})$.
If $M_1, \ldots, M_k$ is a local minimum of $L_k$, then $M_{j+1}, \ldots, M_k$ must be a local minimum of $\wt{L}_B$.
In addition, $d_j = \min_{0\leq i \leq k} d_i =\min_{j \leq i \leq k} d_i$, so there is
no bottleneck in $M_{j+1}, \ldots, M_k$. We apply Theorem~\ref{thm:all} to get that
$A=M_k\cdots M_{j+1}$ is a global minimum of $g_B$,
hence $\pder{L_2}{A}(A,B)=0$. Similarly $\pder{L_2}{B}(A,B)=0$.
\end{proof}

Now we prove the key technical claim of Lemma~\ref{lem:full}.
\begin{proof}\textbf{(Lemma~\ref{lem:full})}

We just need to show that if $f'(AB)\neq 0$, then $A,B$ must be of rank $d$.
Below we assume $f'(AB)\neq 0$. We will show that $A$ has rank $d$. For $B$, we can apply
the same argument to $g(X)=f(X^T)$.

Let $r$ denote the rank of $A$. We will derive contradiction by
assuming $r<d$. We first use an argument in~\cite{laurent2018}
to construct a family of local minima. Since $A=M_k\cdots M_{j+1}$ is
of rank $r<d$, for any $2\leq i\leq j+1$, $M_k \cdots M_i$ has rank at
most $r$. Since $r<d\leq d_{i-1}$, there exists nonzero $w_{i-1}\in \R^{d_{i-1}}$
such that $M_k \cdots M_i w_{i-1}=0$. Then for any $v_{i-1}\in \R^{d_{i-2}}$, we
have
\[M_k\cdots M_i(M_{i-1} + w_{i-1} v_{i-1}^T) = M_k \cdots M_{i-1}\,.\]

Now for any $v_1, v_2, \cdots, v_j$ where $v_i \in \R^{d_{i-1}}$, we claim that
\begin{equation}\label{eq:1}
M_k \cdots M_{j+1} (M_j + w_j v_j^T) \cdots (M_1 + w_1 v_1^T) = M_k \cdots M_1\,.
\end{equation}

This can be shown inductively for $i=j, \cdots, 1$.
\[M_k \cdots M_{j+1} (M_j + w_j v_j^T) \cdots (M_i + w_i v_i^T) = M_k \cdots M_i\,\]

Since $M=(M_1, \ldots, M_k)$ is a local minimum, it is the minimum in an open neighborhood of $M$. If we set $\|v_i\|$'s small enough so that $\wt{M} = (M_1+w_1 v_1^T, \cdots, M_j+w_j v_J^T, M_{j+1}, \cdots, M_k)$ is in a smaller neighborhood, then $\wt{M}$ is a local minimum too since $L_k(\wt{M}) = L_k(M)$. See Claim~1 in~\cite{laurent2018}
for a rigorous proof. 

Let $\wt{B} = \wt{M}_j\cdots \wt{M}_1$. Then by Lemma~\ref{lem:critical},
$\pder{L_2}{A}(A,\wt{B})=0$, i.e $\pder{f(AB)}{A}(A, \wt{B})=f'(A\wt{B})\wt{B}^T=0$.
Since $A\wt{B}=AB$, we have that for any $\wt{M}_1, \ldots, \wt{M}_j$ constructed above,
\begin{equation}\label{eq:x}
f'(AB)\wt{B}^T=0\,.
\end{equation}

For any matrix $M\in\R^{m_1\times m_2}$, denote by $M^{\ell}\in\R^{m_2}$ the $\ell$-th row vector
of $M$, and by $R(M)$ all the row vectors of $M$. Consider the linear subspace
\[V=\{v\in \R^{d_0}\,|\, f'(AB)v=0\,\}\,.\]

Then (\ref{eq:x}) implies that $R(\wt{B})\subseteq V$. We now show that we
can choose $v_i$'s for $1\leq i\leq j$, such that $\wt{B} = \wt{M}_j\cdots \wt{M}_1$
contains a row vector which is not in $V$ to reach a contradiction.

Let $i^\ast = \min \{i\,|\, R(M_i \cdots M_1) \subseteq V\}$. If $i^\ast=1$,
we choose a sufficiently small non-zero vector $v_1 \notin
V$. This can be done by our assumption that $f'(AB)\neq 0$.
Set $\wt{M}_1 = M_1 + w_1 v_1^T$.  Since $w_1\neq 0$, there exists $\ell$ such that
$w_{1\ell}\neq 0$. Then $\wt{M}_1^\ell=M_1^\ell + w_{1\ell} v_1$. By
$M_1^\ell\in V, v_1\notin V, w_{1\ell}\neq 0$, we have
$\wt{M}_1^\ell\notin V$ since $V$ is a linear subspace. Now assuming $i^\ast>1$. Suppose that
we have constructed $\wt{M}_i, \ldots, \wt{M}_1$, for some $i\geq i^\ast-1\geq 1$,
such that $R(\wt{M}_i\cdots\wt{M}_1)\nsubseteq V$. We show the construction
for $i+1$. If $R(M_{i+1}\wt{M}_i\cdots\wt{M}_1)\nsubseteq V$, then
we can simply set $\wt{M}_{i+1} = M_{i+1}$. Assume below that $R(M_{i+1}\wt{M}_i\cdots\wt{M}_1)\subseteq V$.
By inductive hypothesis $R(\wt{M}_i\cdots\wt{M}_1)\nsubseteq V$,
thus there exists say the $\ell$-th row vector $v$ of $\wt{M}_i\cdots\wt{M}_1$ not in $V$. 
Set $v_{i+1}$ as the $\ell$-th basis vector in $R^{d_i}$ so that $v_{i+1}^T \wt{M}_i\cdots\wt{M}_1=v^T$.
Now let $\wt{M}_{i+1} = M_{i+1} + w_{i+1}v_{i+1}^T$. Then
\begin{align*}
\wt{M}_{i+1}\wt{M}_i\cdots\wt{M}_1 &= M_{i+1}\wt{M}_i\cdots\wt{M}_1 + w_{i+1}v_{i+1}^T \wt{M}_i\cdots\wt{M}_1\\
&=M_{i+1}\wt{M}_i\cdots\wt{M}_1 + w_{i+1} v^T\,.
\end{align*}

Since $R(M_{i+1}\wt{M}_i\cdots\wt{M}_1)\subseteq V$
but $v\notin V$ and $w_{i+1}\neq 0$, by the same argument as for $i^\ast=1$, there
must exist a row vector in $\wt{M}_{i+1}\cdots\wt{M}_1$ which is not in
$V$. We have inductively constructed $\wt{B}=\wt{M}_j\cdots\wt{M_1}$
such that $R(\wt{B})\nsubseteq V$, contradicting to (\ref{eq:x}). Hence
$A$ must have rank $d$. This concludes the proof.
\end{proof}

Corollary~\ref{cor:square} immediately follows from Theorem~\ref{thm:main} and~\cite{zhou2018} (Theorem~2(1)). In the following, we prove Corollary~\ref{cor:multi}.
\begin{proof}\textbf{(Corollary~\ref{cor:multi})\quad}
If some tower has no bottleneck, we can fix all the parameters but this tower, we can then apply Theorem~\ref{thm:main} to show that
any local minimum is a global minimum. Hence the statement holds. Below we assume that each tower has a bottleneck with width $b_i$.

Suppose that we have a spurious local minimum $\M = (M_{11}, \cdots, M_{1k_1}, \cdots, M_{s1}, \cdots, M_{sk_s})$. Similar to
the proof of Theorem~\ref{thm:main}, we can break each tower $i$ as $A_i\in \R^{m\times b_i},B_i\in \R^{b_i \times n}$ at the bottleneck layer. Write $M=\sum A_i B_i$. Similarly we can show that either $f'(M)=0$ or all the $A_i, B_i$'s have full rank. Since $\M$ is a spurious local minimum, it must be the second case, i.e. all the $A_i, B_i$ have full rank. Now let $A=(A_1, \cdots, A_s)$ and $B=\begin{pmatrix} B_1 \\ \vdots \\ B_s \end{pmatrix}$. Then $A\in\R^{m\times b},B\in\R^{b\times n}$ where $b=b_1+\cdots+b_s$. By the same argument in the proof of Theorem~\ref{thm:main}, any perturbation of $A,B$ can be done through perturbation of $\M$, hence $A,B$ is a spurious local minimum for the single tower two layer network $AB$. If $b \geq m, n$, then Theorem~\ref{thm:all} implies that any local minimum is a global minimum of $f$.
\end{proof}

\section{Conclusion}

We have shown a non-degeneracy property of local minima in deep linear networks for general convex differentiable loss. This property
allows us to reduce the existence of spurious local minima in a deep (with depth $\geq 3$) linear network to the two layer linear network,
and, for two layer networks, to simplify analysis by removing the rank-deficient case. We show the
application to quadratic loss functions and the generalization to multi-tower deep linear networks. Our proof uses a novel perturbation argument and does not require any heavy mathematical machinery.

It would be interesting to study when there is no spurious local minima beyond the quadratic loss. By our result, we only need to consider the
two layer case. Another interesting question is whether similar phenomenon exists for non-linear networks.

\arxiv{\paragraph{\large{Acknowledgments}} The author would like to thank Walid Krichene for the careful review and suggestions to improve the presentation.}

\bibliography{dln}

\begin{thebibliography}{10}

\bibitem{x3}
Z.~Allen-Zhu, Y.~Li, and Y.~Liang.
\newblock Learning and generalization in overparameterized neural networks,
  going beyond two layers.
\newblock {\em arXiv preprint arXiv:1811.04918}, 2018.

\bibitem{x4}
Z.~Allen-Zhu, Y.~Li, and Z.~Song.
\newblock A convergence theory for deep learning via over-parameterization.
\newblock {\em arXiv preprint arXiv:1811.03962}, 2018.

\bibitem{arora2018a}
S.~Arora, N.~Cohen, N.~Golowich, and W.~Hu.
\newblock A convergence analysis of gradient descent for deep linear neural
  networks.
\newblock In {\em International Conference on Learning Representations}, 2019.

\bibitem{arora18}
S.~Arora, N.~Cohen, and E.~Hazan.
\newblock On the optimization of deep networks: Implicit acceleration by
  overparameterization.
\newblock {\em arXiv}, 2018.

\bibitem{baldi1989}
P.~Baldi and K.~Hornik.
\newblock Neural networks and principal component analysis: Learning from
  examples without local minima.
\newblock {\em Neural networks}, 2(1):53--58, 1989.

\bibitem{bartlett2018}
P.~Bartlett, D.~Helmbold, and P.~Long.
\newblock Gradient descent with identity initialization efficiently learns
  positive definite linear transformations.
\newblock In {\em ICML}, 2018.

\bibitem{x5}
S.~S. Du, J.~D. Lee, L.~W. Haochuan Li~and, and X.~Zhai.
\newblock Gradient descent finds global minima of deep neural networks.
\newblock {\em arXiv preprint arXiv:1811.03804}, 2018.

\bibitem{x7}
S.~S. Du, J.~D. Lee, H.~Li, L.~Wang, and X.~Zhai.
\newblock Gradient descent finds global minima of deep neural networks.
\newblock {\em arXiv preprint arXiv:1811.03804}, 2018.

\bibitem{hardt2017}
M.~Hardt and T.~Ma.
\newblock Identiy matters in deep learning.
\newblock In {\em International Conference on Learning Representations}, 2017.

\bibitem{kawaguchi2016}
K.~Kawaguchi.
\newblock Deep learning without poor local minima.
\newblock In {\em Advances in Neural Information Processing Systems}, pages
  586--594, 2016.

\bibitem{laurent2018}
T.~Laurent and J.~von Brecht.
\newblock Deep linear networks with arbitrary loss: All local minima are
  global.
\newblock In {\em International Conference on Machine Learning}, pages
  2908--2913, 2018.

\bibitem{x2}
Y.~Li and Y.~Liang.
\newblock Learning overparameterized neural networks via stochastic gradient
  descent on structured data.
\newblock {\em NeurIPS}, 2018.

\bibitem{lu2017depth}
H.~Lu and K.~Kawaguchi.
\newblock Depth creates no bad local minima.
\newblock {\em arXiv preprint arXiv:1702.08580}, 2017.

\bibitem{x1}
M.~Soltanolkotabi, A.~Javanmard, and J.~D. Lee.
\newblock Theoretical insights into the optimization landscape of
  over-parameterized shallow neural networks.
\newblock {\em IEEE Transactions on Information Theory}, 2018.

\bibitem{venturi2018spurious}
L.~Venturi, A.~S. Bandeira, and J.~Bruna.
\newblock Spurious valleys in two-layer neural network optimization landscapes.
\newblock {\em arXiv preprint arXiv:1802.06384}, 2018.

\bibitem{yun2018}
C.~Yun, S.~Sra, and A.~Jadbabaie.
\newblock Global optimality conditions for deep neural networks.
\newblock In {\em International Conference on Learning Representations}, 2018.

\bibitem{yun2019}
C.~Yun, S.~Sra, and A.~Jadbabaie.
\newblock Small nonlinearities in activation functions create bad local minima
  in neural networks.
\newblock In {\em International Conference on Learning Representations}, 2019.

\bibitem{zhang2004solving}
T.~Zhang.
\newblock Solving large scale linear prediction problems using stochastic
  gradient descent algorithms.
\newblock In {\em ICML}, 2004.

\bibitem{zhou2018}
Y.~Zhou and Y.~Liang.
\newblock Critical points of linear neural networks: Analytical forms and
  landscape properties.
\newblock In {\em International Conference on Learning Representations}, 2018.

\bibitem{x6}
D.~Zou, Y.~Cao, D.~Zhou, and Q.~Gu.
\newblock Stochastic gradient descent optimizes over-parameterized deep relu
  networks.
\newblock {\em arXiv preprint arXiv:1811.08888}, 2018.

\end{thebibliography}
\icml{\bibliographystyle{icml2019}}
\arxiv{\bibliographystyle{abbrv}}
\nips{\bibliographystyle{abbrv}}

\end{document}